%% file: main.tex
\documentclass{article}

\usepackage{arxiv}

\usepackage[utf8]{inputenc}
\usepackage[T1]{fontenc}
\usepackage{amsmath}
\usepackage{amsfonts}
\usepackage{amssymb}
\usepackage{amsthm}
\usepackage{mathtools}
\usepackage{booktabs}
\usepackage{graphicx}
\usepackage{subcaption}
\usepackage{algorithm}
\usepackage{algorithmic}
\usepackage{tikz}
\usepackage{pgfplots}
\usepackage{float}
\usepackage{xurl}
\usepackage[numbers,sort&compress]{natbib}
\usepackage[hidelinks]{hyperref}
\usepackage{microtype}

\pgfplotsset{compat=1.18}
\usetikzlibrary{arrows.meta,positioning,fit,calc}

\newtheorem{proposition}{Proposition}
\newtheorem{lemma}{Lemma}

\title{HS-FNO: History-Space Fourier Neural Operator for Non-Markovian Partial Differential Equations}

\author{Lennon J. Shikhman\thanks{Corresponding author. ORCID: \href{https://orcid.org/0009-0008-6030-7641}{0009-0008-6030-7641}.} \\
College of Computing, Georgia Institute of Technology \\
Atlanta, GA, United States \\
\texttt{lshikhman3@gatech.edu}}

\renewcommand{\headeright}{Preprint}
\renewcommand{\undertitle}{Preprint}
\renewcommand{\shorttitle}{HS-FNO for Non-Markovian PDEs}

\hypersetup{
pdftitle={HS-FNO: History-Space Fourier Neural Operator for Non-Markovian Partial Differential Equations},
pdfauthor={Lennon J. Shikhman},
pdfkeywords={delay partial differential equations, neural operators, Fourier neural operators, history spaces, non-Markovian dynamics, scientific machine learning}
}

\begin{document}
\maketitle

\begin{abstract}
Neural operators provide fast surrogate models for time-dependent partial differential equations, but their standard autoregressive use usually assumes that the instantaneous field \(u(t,\cdot)\) is a complete state. This assumption is false for delay equations, distributed-memory systems, and other non-Markovian dynamics: two trajectories may agree at time \(t\) and nevertheless have different futures because their histories differ. We introduce the History-Space Fourier Neural Operator (HS-FNO), a neural operator for delay and memory-driven PDEs formulated on the lifted state \(u_t(\theta,x)=u(t+\theta,x)\), \(\theta\in[-\tau,0]\). The method decomposes one step of evolution into an FNO predictor for the newly exposed future slice and an exact shift-append transport for the part of the history window already known from the previous state. This construction reduces the learned output dimension and enforces the natural history-space update. We test HS-FNO on five benchmark families covering delayed reaction--diffusion, spatial epidemiology, nonlocal neural-field dynamics, delayed waves, and distributed-memory closures. Across ten random seeds, HS-FNO attains the lowest aggregate one-step, history-space, and rollout errors among the baselines. The largest gain occurs in autoregressive prediction, where rollout error decreases from \(0.241\), \(0.188\), and \(0.185\) for current-state, lag-stack, and unconstrained history-to-history operators, respectively, to \(0.094\). The same model uses fewer parameters than unconstrained history prediction. These results indicate that representing the history segment as the state, rather than as ordinary extra channels, is an effective inductive bias for non-Markovian PDE surrogate modeling.
\end{abstract}


\input{sections/1_Introduction}
\input{sections/2_RelatedWork}
\input{sections/3_HSFNO}
\input{sections/4_Theory}
\input{sections/5_Experiments}
\input{sections/6_Results}
\input{sections/7_Discussion}
\input{sections/8_Conclusion}

\section*{Data and code availability}
The code and experiment artifacts for this study are available at \url{https://github.com/lennonshikhman/hs-fno}. The repository includes the implementation of HS-FNO, baseline models, benchmark-generation scripts, training and evaluation code, and figure-generation utilities needed to reproduce the reported experiments.

\section*{Acknowledgements}
The author gratefully acknowledges hardware support from Dell Technologies through the Dell Pro Precision division, which supported the computational experiments reported in this work.

\bibliographystyle{unsrtnat}
\bibliography{references}

\end{document}

%% file: sections/1_Introduction.tex
\section{Introduction}

Many scientific and engineered systems evolve according to both their present state and their past states. Delay differential equations (DDEs) and delay partial differential equations (DPDEs) capture this dependence by making the dynamics depend on a history segment rather than only on the instantaneous state \cite{hale1993functional,wu1996partialfunctional}. This structure arises in several computational settings. In population and biological dynamics, delayed reaction--diffusion equations model latency, maturation, and feedback, including threshold behavior in diffusive Nicholson-type equations \cite{yi2009threshold,wu1996partialfunctional}. In chemical kinetics, delayed variables compactly represent intermediate reaction pathways without resolving every species \cite{roussel1996use}. In controlled or engineered spatial systems, delays enter through boundary feedback, actuator latency, sensor latency, and communication effects, where even small nonlinear boundary delays require separate well-posedness analysis \cite{oliva1999reaction}. Memory also appears when unresolved variables are eliminated from multiscale systems: in coarse-grained turbulence and large eddy simulation, Mori--Zwanzig reductions induce nonlocal-in-time closure terms \cite{parish2017nonmarkov,parish2017dynamic}. In all of these cases, \(u(t,\cdot)\) alone may not determine the future. The history is part of the state.

Neural operators are now widely used to approximate solution maps of parametric PDEs, especially when repeated forward solves are needed. In a standard autoregressive formulation, one learns
\[
G_\theta:u(t,\cdot)\mapsto u(t+\Delta t,\cdot),
\]
possibly conditioned on parameters, forcing, or boundary data. This map is well defined for Markovian evolution equations. It is not well defined for delay PDEs unless the state is enlarged. Two trajectories can have identical current fields and different histories, and therefore different next states. A current-state neural operator is then asked to approximate a single-valued map even though the underlying dynamics are not single-valued on instantaneous fields. The problem may be partly hidden in one-step tests drawn from a narrow distribution, but it is exposed by autoregressive rollout, delay shifts, parameter shifts, and feedback regimes where old prediction errors re-enter the dynamics.

This paper studies neural-operator surrogate modeling after replacing the instantaneous state by the natural history state. For delay horizon \(\tau>0\) and spatial state space \(X\), define
\[
u_t(\theta,x)=u(t+\theta,x),\qquad \theta\in[-\tau,0].
\]
The evolution over one time step is then a map
\[
S_{\Delta t}:u_t\mapsto u_{t+\Delta t}
\]
on the history space. We introduce the History-Space Fourier Neural Operator (HS-FNO), which applies an FNO backbone to the lifted history field but does not learn the entire history-to-history map. Instead, HS-FNO predicts only the newly exposed future slice and advances the rest of the window by exact shift-append transport. Thus the learned component is restricted to the part of the update that is not already known from the previous history window.

The numerical question is whether this structural decomposition improves accuracy and stability relative to current-state, lag-stack, unconstrained history-to-history, and sequence baselines. We evaluate this question on five benchmark families: delayed reaction--diffusion, spatial epidemiology, nonlocal neural-field dynamics, delayed waves, and distributed-memory closures. The experiments use matched train/test splits, ten random seeds, rollout evaluation, held-out delay and parameter regimes, and resolution-transfer tests.

The contributions are as follows:
\begin{itemize}
    \item a formulation of delay-PDE surrogate modeling as operator learning on history spaces rather than instantaneous spatial fields;
    \item HS-FNO, a History-Space Fourier Neural Operator that combines a learned future-slice predictor with exact shift-append history transport;
    \item a numerical evaluation across five delay and memory-driven PDE benchmark families, including one-step, history-space, rollout, efficiency, and ablation metrics.
\end{itemize}

The remainder of the paper is organized as follows. Section~\ref{sec:related} reviews delayed dynamics, neural operators, and history-aware prediction. Section~\ref{sec:hsfno} defines HS-FNO. Section~\ref{sec:theory} analyzes the shift-append inductive bias. Sections~\ref{sec:experiments} and~\ref{sec:results} present the numerical setup and results, and Section~\ref{sec:discussion} discusses limitations and future directions.

%% file: sections/2_RelatedWork.tex
\section{Related Work}
\label{sec:related}

\paragraph{DPDEs and numerical solvers}
DPDEs, also called partial functional differential equations, model spatially distributed systems whose evolution depends on past states. Their natural state is a history segment \(u_t\in C([-\tau,0];X)\), \(u_t(\theta)=u(t+\theta)\), rather than only the instantaneous field \(u(t)\in X\). This history-space viewpoint underlies classical semigroup and semiflow theory for delayed reaction--diffusion, wave, population, control, climate, and neutral functional evolution equations \cite{hale1993functional,wu1996partialfunctional}. It also changes the numerical problem. A solver must maintain a moving history window, evaluate delayed terms by interpolation or lookup, and couple this memory with the spatial discretization. Classical approaches include method-of-lines reductions to delay ODEs, finite-difference schemes, finite-element or Galerkin methods, spectral or collocation methods, and delay-specific time stepping \cite{sorokon2022nonlinearrde,pao2001finitedifference,wang2003asymptotic}. Stability, convergence, monotone-iteration, and asymptotic results are available for several delayed reaction--diffusion and pseudoparabolic classes \cite{wang2006timeDelayed,amirali2014explicit,lubo2022galerkin}. Spatially nonlocal delay PDEs can require specialized FDTD schemes when standard method-of-lines reductions are not directly applicable \cite{fang2019fdtd}. Our goal is not to replace these reference solvers, but to use their state-space structure to define an amortized surrogate over histories and delays.

\paragraph{Machine learning for delayed and non-Markovian dynamics}
Learning-based methods for delayed equations include neural solvers, delay-aware dynamical-system architectures, and delay-identification methods. Physics-informed neural networks and physics-informed machine learning incorporate differential-equation structure into training objectives \cite{raissi2019pinn,karniadakis2021piml}. For DPDEs, Huang and Zhu \cite{huang20225double} proposed a double-activation PINN-style network with piecewise fitting for derivative discontinuities in parabolic equations with time delay, while Wang et al. \cite{wang2026dnn} developed a PINN-DDE framework for forward and inverse DDE problems and delay-parameter estimation. Feng et al. \cite{feng2024highdim} used Fourier neural operators for high-dimensional time-delay chaotic systems by mapping one history segment to the next. Statistics-informed neural networks provide another relevant JCP example: Zhu et al. \cite{zhu2023sinn} learned stochastic Markovian and non-Markovian dynamics by matching statistical behavior motivated by projection-operator modeling. A broader neural-delay literature includes Neural DDEs with delayed adjoints, piecewise-constant delay extensions, and stabilized NDDEs for partially observed or delayed dynamics \cite{zhu2021neural,Zhu2022NeuralPCDDE,Schlaginhaufen2021stabledeep}. Recent work also learns delays rather than assuming them fixed, using adjoint sensitivity or adaptive-delay optimal-control formulations \cite{stephany2024learning,zhou2025nadde}. HS-FNO differs in its target problem: it learns a reusable spatial solution operator \(S_{\Delta t}:u_t\to u_{t+\Delta t}\) on history-field states.

\paragraph{Neural operators for PDE surrogate modeling}
Neural operators learn maps between function spaces and have become a central architecture class for PDE surrogate modeling. DeepONet uses branch--trunk networks motivated by nonlinear operator approximation, while Fourier Neural Operators (FNOs) use spectral convolution kernels for parametric PDE solution maps across discretizations \cite{lu2021deeponet,li2021fno}. Subsequent work has developed general operator-learning theory, geometry-informed operators for irregular domains, and convolutional or U-shaped neural operators \cite{kovachki2023neuraloperator,li2023gino,raonic2023cno}. Other variants include neural-field operators, transformer-based operators, latent neural operators, and physics-informed neural operators \cite{rahman2023uno,serrano2023neuralfields,cao2021galerkintransformer,li2023pdeformer,wang2024latentoperator,li2023pino}. Derivative-informed neural operators (DINO) show that adding derivative information can improve operator and Jacobian approximation in high-dimensional parametric settings, a JCP example of using problem structure to strengthen neural-operator surrogates \cite{olearyroseberry2024dino}. Benchmarks such as PDEBench have standardized evaluation across equations, parameters, and initial conditions \cite{takamoto2022pdebench}. Most time-dependent neural-operator benchmarks are Markovian: the learned map is from \(u(t,\cdot)\) to \(u(t+\Delta t,\cdot)\). DPDEs break this setup because the same current field can correspond to multiple valid futures. HS-FNO keeps the FNO backbone but changes the domain of the learned evolution from instantaneous fields to lifted history fields.

\paragraph{History-aware prediction}
The use of history for prediction has a deep mathematical foundation. Delay-coordinate embedding reconstructs state from past observations. Takens \cite{takens1981strange} established generic reconstruction from \(2m+1\) delayed scalar observations for compact \(m\)-dimensional systems, while Sauer et al. \cite{sauer1991embedology} extended this to fractal attractors using box-counting dimension. A complementary explanation comes from projection and coarse-graining. Zwanzig \cite{Zwanzig1961memory} derived exact non-Markovian kinetic equations with memory functions, and Chorin et al. \cite{chorin2000optimalpred} connected Mori--Zwanzig memory to optimal prediction for underresolved dynamics. Later work shows that memory kernels improve coarse-grained or reduced models when time-scale separation is weak \cite{li2015incorporation,gouasmi2017priori}. These viewpoints explain why lag stacks, recurrent networks, and attention models can help: LSTMs preserve long-lag information through gated memory, ConvLSTMs extend recurrence to spatiotemporal fields, and Transformers use self-attention to connect sequence positions \cite{hochreiter1997lstm,shi2015convlstm,vaswani2017attention}. For DPDEs, however, the history field \(u_t(\theta,\cdot)=u(t+\theta,\cdot)\) is not merely a feature representation. It is the state variable on which the evolution is single-valued. This distinction motivates enforcing history transport explicitly rather than asking a generic sequence model to infer it from data.

%% file: sections/3_HSFNO.tex
\section{History-Space Fourier Neural Operator}
\label{sec:hsfno}

This section defines the History-Space Fourier Neural Operator (HS-FNO). The starting point is the lifted history state
\[
    u_t(\theta,x)=u(t+\theta,x), \qquad \theta\in[-\tau,0], \quad x\in\Omega .
\]
For a spatial function space \(X\), the natural delay-PDE state space is
\[
    \mathcal H_\tau=C([-\tau,0];X),
\]
and the exact evolution over a step \(\Delta t\) is
\[
    S_{\Delta t}:\mathcal H_\tau\to\mathcal H_\tau .
\]
A generic neural operator could approximate \(S_{\Delta t}\) as an unconstrained history-to-history map. This is unnecessarily expensive for a one-step update: most of the next history window is already contained in the current window. HS-FNO exploits this structure by combining an FNO predictor for the unknown slice with an exact shift-append update for the known slices.

\paragraph{Operator decomposition}
We decompose the update into a learned future-slice predictor and deterministic history transport. Given history \(u_t\), parameters \(\mu\), delay \(\tau\), and step size \(\Delta t\), HS-FNO predicts only the newly exposed future state:
\begin{equation}
    \widehat{u}(t+\Delta t,\cdot)
    =
    P_\Theta(u_t,\mu,\tau,\Delta t).
\end{equation}
The full history update is assembled as
\begin{equation}
    G_\Theta(u_t;\mu,\tau,\Delta t)
    =
    \operatorname{ShiftAppend}
    \bigl(u_t,\widehat{u}(t+\Delta t,\cdot)\bigr).
\end{equation}
Thus, the neural operator learns the unknown future slice while the known history is transported exactly.

\input{figures/figure1}

\paragraph{Shift-append history evolution}
For \(0<\Delta t\leq \tau\), the exact update satisfies
\begin{equation}
    u_{t+\Delta t}(\theta,\cdot)
    =
    u_t(\theta+\Delta t,\cdot),
    \qquad
    \theta\in[-\tau,-\Delta t].
\end{equation}
Only the newly exposed part of the window must be predicted. On a discrete grid
\[
    \theta_j=-\tau+j\delta_\theta,\qquad j=0,\ldots,M,
\]
with \(\Delta t=m\delta_\theta\), the shift-append update is
\begin{equation}
    G_\Theta(u_t)(\theta_j,\cdot)
    =
    \begin{cases}
        u_t(\theta_{j+m},\cdot), & 0\leq j\leq M-m,\\
        \widehat{u}(t+\theta_j+\Delta t,\cdot), & M-m<j\leq M .
    \end{cases}
\end{equation}
For \(m=1\), HS-FNO predicts one future spatial slice; for \(m>1\), it predicts a short future segment or recursively applies the one-step update.

\paragraph{Fourier neural operator predictor}
The predictor \(P_\Theta\) is an FNO over the history-space domain:
\[
    P_\Theta:(u_t,\mu,\tau,\Delta t)
    \mapsto
    \widehat{u}(t+\Delta t,\cdot).
\]
For one-dimensional spatial problems, \(u_t\) is a two-dimensional field over \((\theta,x)\); for two-dimensional spatial problems, it is a three-dimensional field over \((\theta,x_1,x_2)\). Spectral convolutions over this lifted domain couple history-time and physical space while preserving the resolution-transfer structure of neural operators. Conditioning variables such as \(\mu\), \(\tau\), and \(\Delta t\) are injected as constant channels, coordinate embeddings, or modulation variables. U-Net and transformer history-space backbones are evaluated as ablations; the default HS-FNO uses the FNO backbone.

\paragraph{Training objective}
The supervised data term compares the predicted history update with the reference solver:
\begin{equation}
    \mathcal{L}_{\mathrm{data}}
    =
    \left\|
    G_\Theta(u_t;\mu,\tau,\Delta t)
    -
    S_{\Delta t}(u_t)
    \right\|_{\mathcal H_\tau}^{2}.
\end{equation}
Because the shifted part is exact, this chiefly penalizes the newly predicted slice or segment and its contribution to the assembled next history.

For long-horizon training, we optionally evaluate an autoregressive rollout loss:
\begin{equation}
    \mathcal{L}_{\mathrm{rollout}}
    =
    \sum_{k=1}^{K}
    w_k
    \left\|
    G_\Theta^{k}(u_t;\mu,\tau,\Delta t)
    -
    S_{k\Delta t}(u_t)
    \right\|_{\mathcal H_\tau}^{2},
\end{equation}
where \(w_k\geq0\) controls the emphasis on later rollout times. Prediction errors in delay PDEs are written into the future history window and can re-enter the dynamics later.

When variable step sizes are used, we also evaluate an optional semiflow-consistency regularizer:
\begin{equation}
\begin{aligned}
    \mathcal{L}_{\mathrm{semi}}
    =
    \bigl\|&G_\Theta(G_\Theta(u_t;\mu,\tau,s);\mu,\tau,r)\\
    &-G_\Theta(u_t;\mu,\tau,s+r)
    \bigr\|_{\mathcal H_\tau}^{2}.
\end{aligned}
\end{equation}
This term encourages the learned evolution to match \(S_r\circ S_s=S_{r+s}\). In the numerical study it is treated as an ablated regularizer, not as a required component of HS-FNO. The training objective is
\begin{equation}
\begin{aligned}
    \mathcal{L}
    ={}&
    \lambda_{\mathrm{data}}\mathcal{L}_{\mathrm{data}}
    +
    \lambda_{\mathrm{rollout}}\mathcal{L}_{\mathrm{rollout}}\\
    &+
    \lambda_{\mathrm{semi}}\mathcal{L}_{\mathrm{semi}} .
\end{aligned}
\end{equation}

\paragraph{Computational advantage}
An unconstrained history-to-history operator predicts the full field on \([-\tau,0]\times\Omega\), although most of it is already known. With \(M+1\) history slices and \(N_x\) spatial degrees of freedom, the naive output size is \(O(MN_x)\) in one spatial dimension. A one-step HS-FNO predicts only \(O(N_x)\) unknown values and fills the rest by exact transport. This reduces output dimension and final-layer cost by roughly \(M\), while avoiding the need to relearn the deterministic shift.

%% file: figures/figure1.tex
\begin{figure}[h]
\centering
\includegraphics[width=\columnwidth]{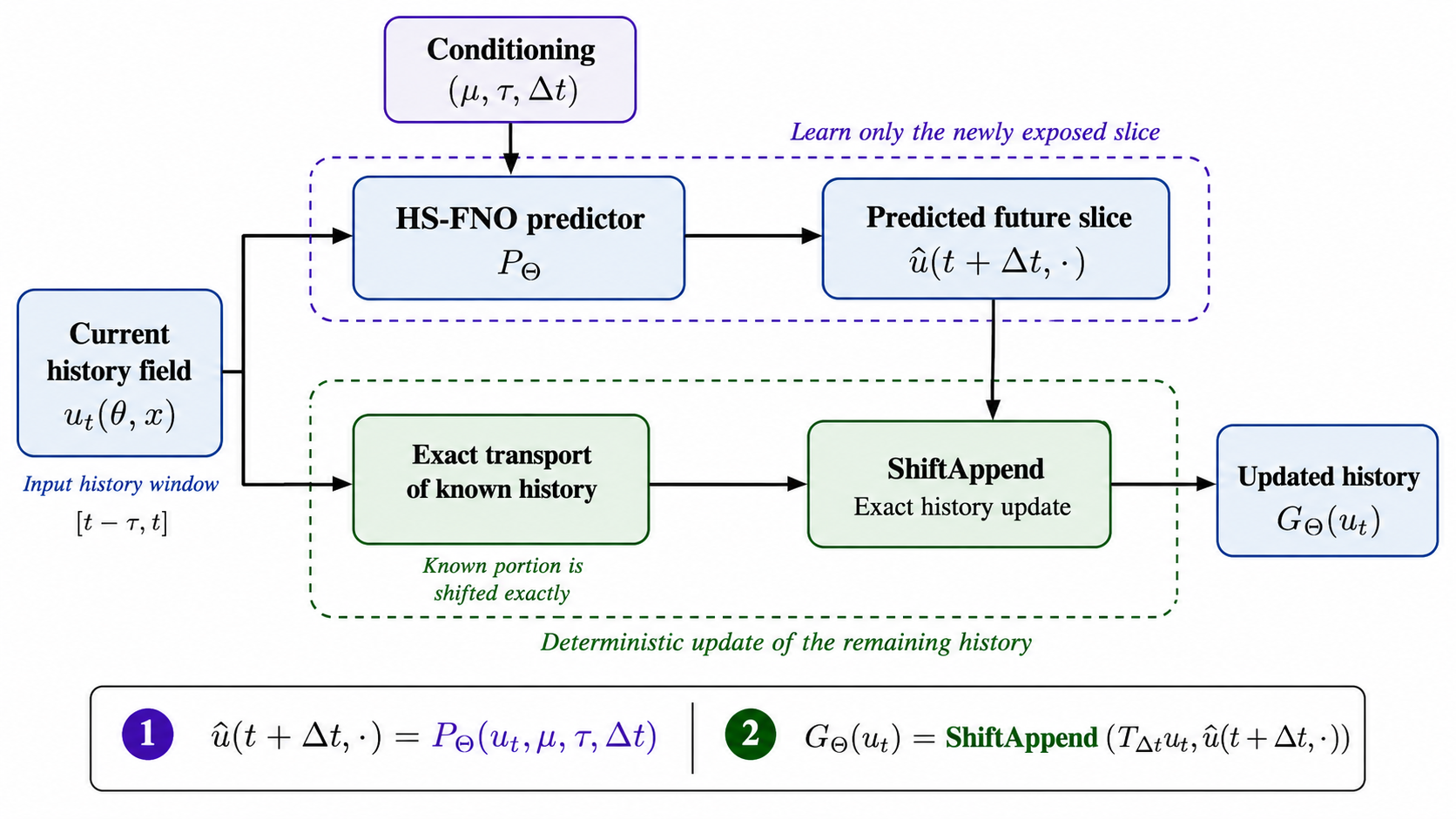}
\caption{HS-FNO architecture. The predictor \(P_\Theta\) receives the current history field \(u_t(\theta,x)\) together with conditioning variables \((\mu,\tau,\Delta t)\) and predicts the newly exposed future slice \(\widehat{u}(t+\Delta t,\cdot)\). The known portion of the history is transported exactly, and ShiftAppend combines this with the predicted slice to form the updated history \(G_\Theta(u_t)\).}
\label{fig:hsfno_architecture}
\end{figure}

%% file: sections/4_Theory.tex
\section{Inductive Bias of History-Space Shift-Append}
\label{sec:theory}

We give a simple finite-dimensional analysis of the shift-append inductive bias. The purpose is not to prove a generalization theorem for FNOs, but to isolate what the architecture removes from the learning problem. We analyze the one-grid-step case \(\Delta t=\delta_\theta\). Let
\(X_N\simeq\mathbb R^n\) with Euclidean norm and
\[
    h=(h_0,\ldots,h_M)\in\mathcal H_N:=X_N^{M+1},
    \qquad h_j\approx u(t-\tau+j\delta_\theta,\cdot).
\]
The exact discrete history update is
\begin{equation}
    S(h)=(h_1,\ldots,h_M,\Phi(h)),
    \label{eq:shiftappend_target}
\end{equation}
where \(\Phi:\mathcal H_N\to X_N\) is the exact next-slice map. Hence only the final
slice is unknown; the first \(M\) output slices are already in the input.

\begin{proposition}[Insufficient history induces irreducible error]
\label{prop:insufficient_history}
Let \(\Pi:\mathcal H_N\to\mathcal Z\) be any history representation. Suppose
\(\Pi(h)=\Pi(h')\) but \(\Phi(h)\neq\Phi(h')\). If, conditional on this common
representation, the histories are \(h\) and \(h'\) with probabilities \(p\in(0,1)\) and
\(1-p\), then every deterministic predictor \(g:\mathcal Z\to X_N\) has conditional
squared error at least
\[
    p(1-p)\|\Phi(h)-\Phi(h')\|_2^2 .
\]
\end{proposition}

\begin{proof}
Any such predictor outputs one value \(z\) for both histories. With
\(y=\Phi(h)\) and \(y'=\Phi(h')\), the risk
\(p\|z-y\|_2^2+(1-p)\|z-y'\|_2^2\) is minimized at
\(z^\star=py+(1-p)y'\), giving \(p(1-p)\|y-y'\|_2^2\).
\end{proof}

Thus current-state and sparse lag-stack models can incur unavoidable conditional error
when their representation identifies histories with different next slices. For
\(Q(h)=(Q_0(h),\ldots,Q_M(h))\), define
\[
    \ell(Q;h)
    =
    \sum_{j=0}^M
    \omega_j\|Q_j(h)-[S(h)]_j\|_2^2,
    \qquad \omega_j\geq 0,
\]
and assume the displayed losses are integrable under \(H\sim\mathcal D\). For a slice
predictor \(\psi\), set
\[
    A_\psi(h)=(h_1,\ldots,h_M,\psi(h)).
\]

\begin{proposition}[Shift-append removes deterministic coordinates]
\label{prop:shift_append}
For \(H\sim\mathcal D\),
\[
\begin{aligned}
\mathbb E[\ell(Q;H)]
&=
\mathbb E\!\left[
\sum_{j=0}^{M-1}
\omega_j\|Q_j(H)-H_{j+1}\|_2^2
\right] \\
&\quad+
\mathbb E\!\left[
\omega_M\|Q_M(H)-\Phi(H)\|_2^2
\right],
\end{aligned}
\]
whereas
\[
    \mathbb E[\ell(A_\psi;H)]
    =
    \mathbb E\!\left[
    \omega_M\|\psi(H)-\Phi(H)\|_2^2
    \right].
\]
\end{proposition}

\begin{proof}
Substitute \eqref{eq:shiftappend_target} into \(\ell\). For \(A_\psi\), all terms
\(j<M\) vanish because \((A_\psi)_j(h)=h_{j+1}\).
\end{proof}

Thus, in the one-step setting, shift-append removes the need to learn \(Mn\)
deterministic output coordinates and leaves only the \(n\)-dimensional next-slice
prediction. An unconstrained history-to-history model can learn this copying operation in
principle; HS-FNO enforces it by construction.

\begin{lemma}[Rollout error under shift-append]
\label{lem:rollout_error}
Let \(h^{k+1}=S(h^k)\), \(\widehat h^{k+1}=A_\psi(\widehat h^k)\), and
\(\widehat h^0=h^0\). Suppose \(\Phi\) is \(L\)-Lipschitz on a set containing the exact
and predicted rollout histories, in the norm
\[
    \|h-\tilde h\|_{\max}:=\max_j\|h_j-\tilde h_j\|_2 .
\]
Assume the slice predictor has uniform on-rollout residual error
\[
    \|\psi(\widehat h^k)-\Phi(\widehat h^k)\|_2\leq\varepsilon .
\]
Then, for \(a_k=\|\widehat h_M^k-h_M^k\|_2\),
\[
    a_{k+1}
    \leq
    \varepsilon
    +
    L\max_{\max\{0,k-M\}\leq r\leq k} a_r .
\]
\end{lemma}

\begin{proof}
Let \(e_j^k=\|\widehat h_j^k-h_j^k\|_2\). Shift-append gives
\(e_j^{k+1}=e_{j+1}^k\) for \(j<M\), so new error is injected only in the final slice.
For that slice,
\[
\begin{aligned}
a_{k+1}
&=
\|\psi(\widehat h^k)-\Phi(h^k)\|_2  \\
&\leq
\|\psi(\widehat h^k)-\Phi(\widehat h^k)\|_2
+
\|\Phi(\widehat h^k)-\Phi(h^k)\|_2 \\
&\leq
\varepsilon+
L\|\widehat h^k-h^k\|_{\max}.
\end{aligned}
\]
Since \(\widehat h^0=h^0\), every slice error in \(\widehat h^k-h^k\) is either zero
from the initial history or a transported current-slice error from one of the previous
\(M\) steps. Hence
\[
    \|\widehat h^k-h^k\|_{\max}
    \leq
    \max_{\max\{0,k-M\}\leq r\leq k}a_r .
\]
Substituting gives the recurrence.
\end{proof}

Lemma~\ref{lem:rollout_error} shows that fresh error enters only through newly predicted
slices. Existing slice errors may still remain in the history window and affect
later predictions through \(\Phi\), so shift-append is not a cure for inaccurate next-slice prediction. Its role is narrower and more concrete: it prevents the model from spending capacity on deterministic copying and confines new errors to the learned part of the update. If \(\psi=\Phi\), then \(A_\psi=S\), so the exact
discrete one-step history update is recovered and its iterates satisfy the discrete
semigroup identity \(S^{k+\ell}=S^k\circ S^\ell\).

%% file: sections/5_Experiments.tex
\section{Experiments}
\label{sec:experiments}

We evaluate HS-FNO on five delay-PDE benchmark families covering delayed local reaction, delayed transmission, nonlocal delayed coupling, delayed wave feedback, and distributed memory. The experiments are designed to test three points: whether the history state improves prediction relative to the instantaneous field, whether exact shift-append transport improves over unconstrained history prediction, and whether these gains persist under rollout, delay shifts, parameter shifts, and resolution transfer.

\subsection{Benchmark Problems}

All benchmarks are generated on a bounded domain \(\Omega\subset\mathbb R^d\), with \(d=1\) for main experiments and \(d=2\) for resolution-transfer tests. Each trajectory starts from an initial history
\[
    \phi(\theta,x)=u(\theta,x), \qquad \theta\in[-\tau,0],
\]
and the learning task is to approximate \(S_{\Delta t}:u_t\mapsto u_{t+\Delta t}\). Parameters, delays, and initial histories are sampled independently per trajectory, and all methods use the same trajectories.

\paragraph{Delayed reaction--diffusion}
We use
\begin{equation}
\begin{aligned}
    \partial_t u(t,x)
    &=
    D\Delta u(t,x)
    +
    r u(t,x)\left(1-u(t-\tau,x)\right), \\
    x&\in\Omega .
\end{aligned}
\label{eq:delayed_rd}
\end{equation}
Here \(u\) is a spatial density, \(D\) is diffusion, \(r\) is growth, and \(\tau\) is delayed population feedback. This abstracts maturation, resource depletion, delayed competition, and diffusive Nicholson-type dynamics \cite{yi2009threshold,wu1996partialfunctional,sorokon2022nonlinearrde}. Initial histories are nonnegative, and negative numerical reference values are clipped only when caused by solver discretization error.

\paragraph{Spatial epidemiology}
For delayed infectiousness or observation lag, we use
\begin{equation}
    \partial_t I(t,x)
    =
    D\Delta I(t,x)
    +
    \beta S(x)I(t-\tau,x)
    -
    \gamma I(t,x).
\label{eq:delayed_epidemic}
\end{equation}
Here \(I\) is infected density, \(S(x)\) is a fixed susceptibility field, \(\beta\) is transmission, and \(\gamma\) is recovery or removal. The susceptibility field is sampled once per trajectory and provided as an input channel.

\paragraph{Nonlocal neural field}
For nonlocal delayed coupling, we consider
\begin{equation}
    \partial_t u(t,x)
    =
    -u(t,x)
    +
    \int_\Omega
    w(x,y)\,
    \sigma\bigl(u(t-\tau(x,y),y)\bigr)
    \,dy .
\label{eq:neural_field}
\end{equation}
The kernel \(w(x,y)\) is smooth and distance-based, \(\sigma\) is nonlinear, and \(\tau(x,y)\) is a positive bounded function of spatial separation. The integral is evaluated by quadrature on the simulation grid.

\paragraph{Delayed wave dynamics}
We include
\begin{equation}
    \partial_{tt}u(t,x)
    =
    c^2\Delta u(t,x)
    +
    F\bigl(u(t-\tau,x)\bigr).
\label{eq:delayed_wave}
\end{equation}
We rewrite this as
\[
    \partial_t u=v,\qquad
    \partial_t v=c^2\Delta u+F(u(t-\tau,x)),
\]
and define the history state over \(z_t=(u_t,v_t)\). This benchmark tests phase accuracy, amplitude preservation, and long-horizon stability under delayed feedback, and is motivated by optical and wave systems requiring explicit delayed history access \cite{fang2019fdtd}.

\paragraph{Distributed-memory dynamics}
Finally, we consider
\begin{equation}
    \partial_t u(t)
    =
    \mathcal{N}(u(t))
    +
    \int_{-\tau}^{0}
    K(\theta)\,
    \mathcal{M}\bigl(u(t+\theta)\bigr)
    \,d\theta .
\label{eq:distributed_memory}
\end{equation}
We instantiate
\[
    \mathcal{N}(u)=\nu\Delta u+f(u),
    \qquad
    \mathcal{M}(u)=a_1u+a_2u^3,
\]
with normalized nonnegative memory kernel \(K\). The memory integral is computed on the model history grid. This generalizes single-lag delay equations and is motivated by Mori--Zwanzig and LES memory closures \cite{Zwanzig1961memory,chorin2000optimalpred,parish2017nonmarkov,parish2017dynamic,gouasmi2017priori}.

\subsection{Data Generation}

Datasets are generated offline with deterministic reference solvers so that all methods are compared against the same numerical trajectories. Local reaction--diffusion, spatial epidemiology, and distributed-memory equations use method-of-steps integration with delayed terms evaluated from a history buffer. The current PDE state is advanced by explicit or semi-implicit time integration. Diffusion uses centered finite differences for Dirichlet or Neumann boundaries and Fourier pseudospectral derivatives for periodic domains. The neural-field integral uses matrix-vector quadrature, and the wave equation is evolved in first-order form. Off-grid delayed times are evaluated by linear interpolation in the reference history.

Each trajectory is specified by \((\phi,\mu,\tau,\mathcal B)\), where \(\phi\) is the initial history, \(\mu\) denotes physical parameters, \(\tau\) is a scalar delay or delay field, and \(\mathcal B\) denotes boundary conditions. Initial histories are smooth random fields using low-frequency Fourier modes or boundary-compatible expansions; density-like fields are shifted and rescaled to be nonnegative. Multi-channel systems use independent channel histories unless constrained by the equation.

Fields are saved at \(t_n=n\Delta t_{\mathrm{save}}\). For delay horizon \(\tau\), the discrete history is
\[
    u_{t_n}
    =
    \{u(t_n+\theta_j,\cdot)\}_{j=0}^{M},
    \qquad
    \theta_j=-\tau+j\delta_\theta,
    \qquad
    \delta_\theta=\tau/M .
\]
Each supervised pair is
\begin{equation}
    (u_{t_n},\mu,\tau,\Delta t)
    \mapsto
    u_{t_n+\Delta t}.
\label{eq:supervised_pair}
\end{equation}
HS-FNO targets only the newly exposed future slice or segment; the full next history is used for history-space losses and history-to-history comparisons. Train, validation, and test sets are split by trajectory. We evaluate in-distribution performance, held-out delays, held-out parameter ranges, and spatial-resolution transfer.

\subsection{Models and Baselines}

All models use the same supervised windows and trajectory splits. The current-state neural operator receives only \(u(t,\cdot)\) and predicts \(u(t+\Delta t,\cdot)\). The lag-stack baseline receives
\[
    [u(t,\cdot),u(t-\delta,\cdot),\ldots,u(t-L\delta,\cdot)]
    \mapsto u(t+\Delta t,\cdot),
\]
with lags sampled from the same history window as HS-FNO, but without enforcing history-state evolution. The unconstrained history-to-history baseline receives \(u_t\) and directly predicts \(u_{t+\Delta t}\) using the same history grid and conditioning variables as HS-FNO. Sequence baselines include ConvLSTM, temporal U-Net, and transformer-over-history models.

The default model is HS-FNO: an FNO over the history-space domain that predicts only the newly exposed future slice and advances the window by exact shift-append transport. We also evaluate HS-FNO variants that remove shift-append, remove explicit delay conditioning, use alternative coordinate or FiLM conditioning, change the history length or history resolution, or add rollout and semiflow-consistency losses. To separate the effect of the FNO backbone from the history-space formulation, we additionally test history-space U-Net and transformer variants. When \(\mu\), \(\tau\), or \(\Delta t\) vary, they are supplied to applicable models as conditioning channels or embeddings.

\subsection{Training and Evaluation}

All models are trained on the same supervised windows and evaluated on the same train/validation/test trajectory splits. In one-step training, each model predicts \(t+\Delta t\); in rollout evaluation, predictions are fed back autoregressively for \(K\) steps. HS-FNO and history-to-history models roll out the predicted history window, while current-state and lag-stack baselines update their input buffers by appending predicted instantaneous fields. The loss is
\[
    \mathcal{L}
    =
    \lambda_{\mathrm{data}}\mathcal{L}_{\mathrm{data}}
    +
    \lambda_{\mathrm{rollout}}\mathcal{L}_{\mathrm{rollout}}
    +
    \lambda_{\mathrm{semi}}\mathcal{L}_{\mathrm{semi}},
\]
with terms defined in the HS-FNO method section. Models without variable-step prediction omit \(\mathcal{L}_{\mathrm{semi}}\).

We report one-step relative error,
\[
    E_{\mathrm{one}}
    =
    \frac{\|\widehat{u}(t+\Delta t)-u(t+\Delta t)\|_2}
         {\|u(t+\Delta t)\|_2},
\]
history-space error,
\[
    E_{\mathrm{hist}}
    =
    \frac{\|\widehat{u}_{t+\Delta t}-u_{t+\Delta t}\|_{\mathcal H_\tau}}
         {\|u_{t+\Delta t}\|_{\mathcal H_\tau}},
\]
and \(K\)-step rollout error,
\[
    E_{\mathrm{roll}}
    =
    \frac{1}{K}\sum_{k=1}^{K}
    \frac{\|\widehat{u}(t+k\Delta t)-u(t+k\Delta t)\|_2}
         {\|u(t+k\Delta t)\|_2}.
\]
For variable-step experiments, we also compute
\[
    E_{\mathrm{semi}}
    =
    \frac{
    \|G_\Theta(G_\Theta(u_t;s);r)-G_\Theta(u_t;s+r)\|_{\mathcal H_\tau}
    }{
    \|G_\Theta(u_t;s+r)\|_{\mathcal H_\tau}
    }.
\]
Efficiency metrics include parameter count, inference time per step, output dimension, peak memory use, and speedup over the reference solver. Main results are averaged over ten independent seeds. We report seed-level means with 95\% bootstrap confidence intervals and use paired comparisons over matched seed--benchmark--regime cells when comparing HS-FNO against baselines. Benchmark--regime cells are treated as structured evaluation conditions rather than independent random trials. Application-specific diagnostics, such as amplitude, frequency, phase, and normalized physics metrics, are computed for analysis. They are not used as headline evidence when a diagnostic is poorly scaled or degenerate for a particular benchmark.

\subsection{Ablations}

We ablate exact shift-append transport, rollout and semiflow losses, input history length \(\tau_{\mathrm{input}}\), history resolution \(M\), explicit delay conditioning versus separate per-delay models, and backbone choice among FNO, convolutional/U-shaped, and transformer history-space models.

%% file: sections/6_Results.tex
\section{Results and Analysis}
\label{sec:results}

We evaluate all models on five benchmark families and four regimes: in-distribution, held-out delay, held-out parameter, and resolution transfer. DPDE results are averaged over ten seeds; confidence intervals are 95\% bootstrap percentile intervals over seed-level aggregate means. HS-FNO denotes the default history-space shift-append model with an FNO backbone.

\paragraph{Aggregate accuracy}
Figure~\ref{fig:main_results} summarizes the main aggregate comparison across one-step, history-space, and rollout metrics. HS-FNO is the best principal model on all three aggregate errors. The one-step error decreases from \(0.143\) for current-state prediction, \(0.114\) for lag-stack prediction, and \(0.113\) for History2History to \(0.066\) for HS-FNO. The history-space error is also lowest for HS-FNO, with error \(0.098\) compared with \(0.114\), \(0.123\), and \(0.126\) for lag-stack, current-state, and History2History models, respectively.

\input{figures/figure2}

The largest gains appear under autoregressive rollout, where errors from an incomplete or poorly structured state representation accumulate over repeated prediction. HS-FNO reduces aggregate rollout error from \(0.241\), \(0.188\), and \(0.185\) for current-state, lag-stack, and History2History models, respectively, to \(0.094\). These correspond to paired improvements of about \(50.8\%\), \(38.0\%\), and \(39.9\%\). The comparison with History2History is especially important: both models receive the full history window, but HS-FNO only learns the newly exposed future slice and transports the known history exactly. Thus the improvement is not only due to using history, but also to imposing the correct update structure on the history state.

\paragraph{Rollout stability and regime behavior}
Figure~\ref{fig:rollout_growth} shows how errors grow over autoregressive rollout steps. HS-FNO has the lowest mean error at each evaluated rollout step and reaches final-step error \(0.122\), compared with \(0.256\) for History2History, \(0.262\) for lag-stack, and \(0.336\) for current-state prediction. The separation between HS-FNO and the baselines widens with rollout horizon, which is consistent with the central motivation of the method: in non-Markovian dynamics, errors are written back into the history window and can influence later predictions.

\input{figures/figure3}

The rollout curves also clarify the difference between using additional past fields as input and enforcing history-space evolution. Lag-stack and History2History improve substantially over the current-state baseline, indicating that access to past states matters. However, both remain well above HS-FNO during rollout. This suggests that treating history as extra channels or directly predicting the full next history does not fully exploit the deterministic shift structure of the delay state. HS-FNO is more stable because only the newly exposed slice is learned, while the rest of the history update is fixed by construction.

The improvement is strongest in aggregate but is not uniform across regimes. HS-FNO variants win most cells, though the best variant changes by regime; the no-delay-conditioning variant slightly improves aggregate rollout error to \(0.090\). The clearest failure mode is resolution transfer, especially delayed reaction--diffusion, where the best HS-FNO-family variant is about \(0.352\) and the default is about \(0.394\). This indicates that the history-space inductive bias improves rollout reliability, but does not by itself solve cross-resolution generalization.

\begin{table}[H]
\centering
\small
\setlength{\tabcolsep}{3.2pt}
\begin{tabular}{lcccc}
\hline
Model & Rollout & Time & Params & Memory \\
\hline
HS-FNO no-delay & \textbf{0.090} & 0.0016 & \(1.23{\times}10^5\) & 20.4 \\
HS-FNO & 0.094 & 0.0018 & \(2.19{\times}10^5\) & 22.0 \\
History2History & 0.185 & 0.0018 & \(6.09{\times}10^6\) & 78.5 \\
Lag-stack & 0.188 & 0.0014 & \(1.05{\times}10^6\) & 27.6 \\
Current-state & 0.241 & \textbf{0.0013} & \(6.72{\times}10^5\) & 24.1 \\
Temporal U-Net & 0.352 & 0.0026 & \(1.61{\times}10^6\) & 33.0 \\
ConvLSTM & 0.383 & 0.0028 & \(1.54{\times}10^6\) & 32.5 \\
Transformer & 0.459 & 0.0020 & \(\mathbf{7.69{\times}10^4}\) & \textbf{19.4} \\
\hline
\end{tabular}
\caption{Accuracy--efficiency comparison. Time is inference time per step in seconds; memory is peak memory in MB. Best value in each column is bolded}
\label{tab:efficiency_results}
\end{table}

\paragraph{Efficiency and ablations}
Table~\ref{tab:efficiency_results} shows that the default HS-FNO gives the best rollout accuracy while remaining efficient. It obtains rollout error \(0.094\) with \(2.19{\times}10^5\) parameters, whereas History2History uses \(6.09{\times}10^6\) parameters and obtains rollout error \(0.185\). Thus, directly predicting the full history is not only less structured but also more expensive. Among the compared models, HS-FNO has the second-fewest parameters and the second-smallest memory footprint. The only model with fewer parameters is the transformer baseline, but its rollout error is much larger at \(0.459\). The current-state model is slightly faster per step, but its rollout error is more than twice that of HS-FNO.

Ablations support the main design. Removing shift-append raises rollout error from \(0.094\) to \(0.150\), showing that exact history transport is not just a cosmetic architectural choice. History-space U-Net and transformer variants are worse at \(0.162\) and \(0.228\), respectively, which suggests that the FNO backbone is a strong match for the lifted history-space representation. The rollout--semiflow objective degrades performance to \(0.475\), so it is not used as headline evidence.

\begin{figure}[H]
\centering
\begin{tikzpicture}
\begin{axis}[
    width=\columnwidth,
    height=5.0cm,
    ybar,
    bar width=7pt,
    ymin=0,
    ymax=4.2,
    ylabel={Average MAE},
    symbolic x coords={HS-FNO,Current,Persistence,H2H,Lag},
    xtick=data,
    xticklabel style={font=\scriptsize, rotate=20, anchor=east},
    yticklabel style={font=\scriptsize},
    ylabel style={font=\small},
    enlarge x limits=0.18,
    grid=major,
    grid style={gray!20},
    legend style={
        font=\scriptsize,
        at={(0.5,-0.23)},
        anchor=north,
        draw=none,
        fill=none
    },
]
\addplot coordinates {(HS-FNO,3.359) (Current,3.713) (Persistence,3.804) (H2H,3.858) (Lag,3.862)};
\addlegendentry{METR-LA}

\addplot coordinates {(HS-FNO,1.818) (Current,2.251) (Persistence,2.178) (H2H,2.338) (Lag,2.336)};
\addlegendentry{PEMS-BAY}
\end{axis}
\end{tikzpicture}
\caption{Real-world traffic sanity check under the standard 12-input/12-output protocol. Bars show denormalized average MAE on METR-LA and PEMS-BAY; lower is better}
\label{fig:real_traffic}
\end{figure}

\paragraph{Real-world spatiotemporal sanity check}
Figure~\ref{fig:real_traffic} evaluates the same history-space idea on METR-LA and PEMS-BAY traffic forecasting \cite{li2018dcrnn}. These are not ground-truth DPDE datasets, and the result should not be interpreted as validation of a specific delay-PDE model for traffic. They provide a real spatiotemporal sanity check for the history-space inductive bias. Under the standard 12-input/12-output protocol, HS-FNO attains the lowest average MAE among the tested baselines on both datasets. This supports the broader claim that representing a recent history window as the predictive state can improve forecasting outside controlled synthetic PDE simulations, while still leaving open whether the underlying system is governed by an explicit delay PDE.

\paragraph{Summary}
Overall, the results support three conclusions. First, history-space prediction improves non-Markovian surrogate modeling relative to current-state, lag-stack, and sequence baselines. Second, exact shift-append transport is beneficial because it avoids relearning deterministic history movement and focuses model capacity on the unknown future slice. Third, HS-FNO provides the strongest aggregate accuracy--efficiency tradeoff among the tested models. These conclusions are strongest for aggregate rollout performance, while resolution transfer remains the main DPDE limitation.

%% file: figures/figure2.tex
\begin{figure*}[h]
\centering
\begin{tikzpicture}
\pgfplotstableread[row sep=\\]{
model mean errplus errminus\\
HS-FNO 0.066 0.009 0.009\\
H2H 0.113 0.012 0.012\\
Lag 0.114 0.014 0.013\\
Current 0.143 0.014 0.014\\
U-Net 0.202 0.036 0.035\\
ConvLSTM 0.238 0.031 0.033\\
Trans. 0.324 0.039 0.039\\
}\onestepdata

\pgfplotstableread[row sep=\\]{
model mean errplus errminus\\
HS-FNO 0.098 0.007 0.006\\
H2H 0.126 0.013 0.012\\
Lag 0.114 0.007 0.007\\
Current 0.123 0.007 0.006\\
U-Net 0.144 0.014 0.013\\
ConvLSTM 0.159 0.012 0.013\\
Trans. 0.189 0.017 0.016\\
}\historydata

\pgfplotstableread[row sep=\\]{
model mean errplus errminus\\
HS-FNO 0.094 0.011 0.009\\
H2H 0.185 0.020 0.020\\
Lag 0.188 0.025 0.024\\
Current 0.241 0.023 0.024\\
U-Net 0.352 0.072 0.068\\
ConvLSTM 0.383 0.055 0.057\\
Trans. 0.459 0.052 0.049\\
}\rolloutdata

\begin{axis}[
    width=0.88\textwidth,
    height=5.7cm,
    scale only axis,
    ybar,
    bar width=4.2pt,
    ymin=0,
    ymax=0.56,
    ylabel={Relative error},
    ylabel style={font=\small},
    yticklabel style={font=\scriptsize},
    symbolic x coords={HS-FNO,H2H,Lag,Current,U-Net,ConvLSTM,Trans.},
    xtick=data,
    xticklabel style={font=\scriptsize, rotate=28, anchor=east},
    enlarge x limits=0.12,
    grid=major,
    grid style={gray!20},
    legend columns=3,
    legend style={
        font=\scriptsize,
        at={(0.5,-0.18)},
        anchor=north,
        draw=none,
        fill=none
    },
    error bars/y dir=both,
    error bars/y explicit,
]

\addplot+[
    draw=blue!80!black,
    fill=blue!30,
    error bars/error bar style={black},
]
table[x=model,y=mean,y error plus=errplus,y error minus=errminus] {\onestepdata};
\addlegendentry{One-step}

\addplot+[
    draw=teal!70!black,
    fill=teal!35,
    error bars/error bar style={black},
]
table[x=model,y=mean,y error plus=errplus,y error minus=errminus] {\historydata};
\addlegendentry{History}

\addplot+[
    draw=red!75!black,
    fill=red!35,
    error bars/error bar style={black},
]
table[x=model,y=mean,y error plus=errplus,y error minus=errminus] {\rolloutdata};
\addlegendentry{Rollout}

\end{axis}
\end{tikzpicture}
\caption{Aggregate relative errors across benchmark--regime cells. Bars show ten-seed means and error bars show 95\% bootstrap confidence intervals. Lower is better.}
\label{fig:main_results}
\end{figure*}

%% file: figures/figure3.tex
\begin{figure*}[h]
\centering
\begin{tikzpicture}
\begin{axis}[
    width=0.88\textwidth,
    height=5.7cm,
    scale only axis,
    xmin=0, xmax=3,
    ymin=0, ymax=0.65,
    xtick={1,2,3},
    xlabel={Rollout step},
    ylabel={Mean relative error},
    xlabel style={font=\small},
    ylabel style={font=\small},
    ticklabel style={font=\scriptsize},
    grid=major,
    grid style={gray!20},
    legend columns=4,
    legend style={
        font=\scriptsize,
        at={(0.5,-0.18)},
        anchor=north,
        draw=none,
        fill=none,
        /tikz/every even column/.append style={column sep=0.3cm}
    },
    line cap=round,
    line join=round,
]
\addplot[very thick, blue, mark=*, mark size=2.6pt, mark options={fill=blue}]
coordinates {(0,0) (1,0.0655) (2,0.0936) (3,0.1218)};
\addlegendentry{HS-FNO}

\addplot[thick, red!85!black, mark=square*, mark size=2.4pt]
coordinates {(0,0) (1,0.1129) (2,0.1843) (3,0.2557)};
\addlegendentry{H2H}

\addplot[semithick, brown!70!black, dashed, mark=triangle*, mark size=2.4pt]
coordinates {(0,0) (1,0.1142) (2,0.1882) (3,0.2622)};
\addlegendentry{Lag}

\addplot[semithick, black, mark=diamond*, mark size=2.2pt]
coordinates {(0,0) (1,0.1433) (2,0.2398) (3,0.3364)};
\addlegendentry{Current}

\addplot[semithick, gray!70!black, mark=o, mark size=2.2pt]
coordinates {(0,0) (1,0.2017) (2,0.3529) (3,0.5041)};
\addlegendentry{U-Net}

\addplot[semithick, red!60, dashed, mark=x, mark size=2.4pt]
coordinates {(0,0) (1,0.2383) (2,0.3819) (3,0.5255)};
\addlegendentry{ConvLSTM}

\addplot[semithick, brown!55, dashed, mark=+, mark size=3.0pt]
coordinates {(0,0) (1,0.3236) (2,0.4497) (3,0.5758)};
\addlegendentry{Trans.}

\node[anchor=west,font=\scriptsize,blue!80!black] at (axis cs:3,0.1218) {$\leftarrow$ best final-step error};

\end{axis}
\end{tikzpicture}
\caption{Mean rollout error over autoregressive prediction steps. HS-FNO has the lowest error at every rollout step among the main models.}
\label{fig:rollout_growth}
\end{figure*}

%% file: sections/7_Discussion.tex
\section{Discussion and Limitations}
\label{sec:discussion}

\paragraph{Interpretation}
History-space operators are most appropriate when \(u(t,\cdot)\) is not a sufficient state, as in systems with explicit delays, distributed memory, unresolved variables, or observation/control lags. In these settings, current-state surrogates can be structurally misspecified because \(u(t,\cdot)\mapsto u(t+\Delta t,\cdot)\) may not be single-valued. HS-FNO addresses this issue by treating the history window as the state and enforcing deterministic shift-append transport. The numerical results support this design most strongly for aggregate rollout accuracy and for the accuracy--efficiency tradeoff against unconstrained history-to-history prediction.

\paragraph{What the method does not solve}
Shift-append transport does not make the next-slice predictor exact, and it does not remove the usual sources of neural-surrogate error. It only prevents the model from relearning deterministic movement of already observed history slices. If the learned future-slice map is inaccurate, its errors are still written into the history window and can influence later predictions. This is why rollout evaluation is essential for this problem class.

\paragraph{Limitations}
HS-FNO increases input size with the number of history slices and depends on the chosen history grid and delay horizon. If the grid undersamples the relevant memory scale, the lifted state remains incomplete even though shift-append is exact on that grid. State-dependent delays, spatially varying delays, and distributed-memory kernels may require interpolation or quadrature, weakening the clean one-slice update. Empirically, performance is not uniform: explicit delay conditioning is not consistently beneficial, rollout--semiflow training degrades performance, and resolution transfer remains the clearest failure mode. The main DPDE benchmarks are controlled synthetic systems; METR-LA and PEMS-BAY provide real spatiotemporal sanity checks, but they are not DPDE datasets.

\paragraph{Future work}
Future work should study latent history compression, adaptive history grids, learned memory quadrature, uncertainty-aware rollouts, stronger resolution transfer, and extensions to inverse problems, delay identification, parameter estimation, and control. Testing on real delayed spatial datasets remains important, since observed histories may be noisy, partial, irregularly sampled, or governed by unknown memory mechanisms.

%% file: sections/8_Conclusion.tex
\section{Conclusion}

This paper presented HS-FNO, a history-space neural operator for non-Markovian PDE surrogate modeling. The central observation is that delay and memory-driven dynamics are not generally single-valued on the instantaneous field \(u(t,\cdot)\). Their state is the lifted history \(u_t(\theta,x)=u(t+\theta,x)\). HS-FNO uses this representation directly, predicts only the newly exposed future slice, and advances the remaining history by exact shift-append transport.

Across five delay and memory-driven PDE families, HS-FNO improves aggregate one-step, history-space, and rollout accuracy over current-state, lag-stack, unconstrained history-to-history, and sequence baselines. The largest gains occur in autoregressive rollout, where errors from misspecified state representations accumulate over time. The efficiency results show that exact shift-append transport is also computationally useful: the model avoids predicting history coordinates that are already known from the previous window.

The results support history-space operator learning as a practical framework for non-Markovian PDE surrogates. They also identify the main remaining issues. Resolution transfer is still the weakest regime, delay conditioning is not uniformly beneficial, and real delayed spatial datasets remain limited. Addressing these points will require adaptive memory representations, stronger cross-resolution training, and better treatment of uncertain or partially observed histories.